\documentclass[a4paper]{article} 
\usepackage{a4wide}

\usepackage[usenames,dvipsnames]{xcolor}
\usepackage[english]{babel} 
\usepackage[color=cyan]{todonotes}

\usepackage{tikz}
\usetikzlibrary{arrows.meta, patterns}
\usepackage{url}
\usepackage{datetime}
\usepackage{hyperref}
\usepackage[hyphenbreaks]{breakurl}

\usepackage{tcolorbox}
\tcbuselibrary{skins}

\usepackage{amssymb,amsmath,amsthm,amsfonts}
\usepackage{mathtools}
\usepackage[T1]{fontenc}  
\definecolor{color1}{RGB}{0,139,0} 

\numberwithin{equation}{section} 
\newcommand{\E}{\mathcal{E}}

\hypersetup{hidelinks,colorlinks,breaklinks=true,urlcolor=black,citecolor=color1,linkcolor=color1,bookmarksopen=false,pdftitle={Title},pdfauthor={Author}}

\makeatletter
\def\@endtheorem{\endtrivlist}
\makeatother

\newtheoremstyle{abcd}
  {}
  {}
  {\itshape}
  {}
  {\bfseries}
  {.}
  {.5em}
  {}

\theoremstyle{abcd}
\newtheorem{theorem}{Theorem}
\numberwithin{theorem}{section}

\newtheorem{conjecture}[theorem]{Conjecture}

\newtheorem{lemma}[theorem]{Lemma}

\begin{document}

\title{The Effect of Iterativity on Adversarial Opinion Forming}

\author{Konstantinos Panagiotou \and Simon Reisser}
\date{$1^{\text{st}}$ December, 2021}
\maketitle 
\thispagestyle{empty} 

\begin{abstract}
Consider the following model to study adversarial effects on opinion forming. A set of initially selected experts form their binary opinion while being influenced by an adversary, who may convince some of them of the falsehood.
All other participants in the network then take the opinion of the majority of their neighbouring experts. 
Can the adversary influence the experts in such a way that the majority of the network believes the falsehood? Alon et al.~\cite{Alon2015} conjectured that in this context an iterative dissemination process will always be beneficial to the adversary. This work provides a counterexample to that conjecture.
\end{abstract}

\section{Introduction}

Understanding how opinions are formed is as important as ever, as the spread of misinformation becomes more prevalent every day. 
Assume there is some new innovation being either good or bad that is introduced to a group of people who want to form their (binary) opinion about it. Following a key insight by Rogers \cite{Rogers2003}, the opining forming process can be modelled as follows. At first, a small set of so-called early adopters, or experts, forms their opinion about the newly introduced innovation. Afterwards, they disseminate their opinion to all other non-experts in the network.

When looking at that network from the outside an observer wants to infer the quality of the new innovation by observing the opinion of all individuals, but without taking the actual structure of the network into consideration (maybe by doing a poll).
One popular method to achieve this is using the \emph{wisdom of the crowd}.
In this case that corresponds to a simple majority rule, that is, the observer takes the majority of opinions as an estimate.
Wisdom of the crowd has been shown to have a plethora of useful applications in decision making, see e.g.~\cite{Cooke2008, Morgan2014, Oprea2009, Aspinall2010, Budescu2015, Mellers2014}.

Assume furthermore that there is an adversary who can  influence the opinion of some early adopters so as to  falsely convince the observer of the new innovation's quality.
Let us look at some examples.
Consider the so-called Black-Hat ASIN Piggybacking on Amazons Marketplace~\cite{Masters2019}. This is the method of hijacking the listing of an Amazon vendor to sell counterfeit products under the (dis-)guise of a genuine listing. Some customers then buy the real product and some buy the fake one. This results in the vendor to lose profit as well as him getting negative reviews that do not correspond to the actual product. The second example is a newly opened restaurant, that in its opening phase invites food critics to try and rate the restaurant. However, when those critics dine at the restaurant, the restaurant puts in more effort than it would when catering to a regular customer, e.g., by providing better quality food and service. Lastly, consider the common practice of online vendors to buy positive reviews for their products by either giving directly monetary incentives to reviewers or providing them with free products. In particular, on Amazon in certain product categories, like Bluetooth speakers and headphones, ReviewMeta \cite{Noonan2016} finds more than half of reviews to be fake \cite{Dwoskin2018}.

\paragraph{A Model for Opinion Forming}
In the previous examples we saw three different sorts of adversaries: the hijacking seller  influenced negatively the opinions of some customers; the restaurant owner could actively choose which critics to influence; finally, the seller that bought his reviews could select the reviewers as well as guarantee their opinion.
Alon et al.~\cite{Alon2015} introduced a model that implements the ideas outlined above. Given a graph $G=(V,E)$ on $n$ vertices and parameters $0\le \mu<1/2 ,0<\delta\le 1/2$ we define the set of experts as a set $\mathcal{E}\subseteq V$ with the property that $|\mathcal{E}|=\mu |V|=\mu n$. Let $\mathcal{E}$ be furthermore  divided into two subsets: the experts that know the truth $\mathcal{E}_1\subseteq \mathcal{E}$ and the experts that are convinced of the falsehood  $\mathcal{E}_0= \mathcal{E}\setminus \E_1$. The sets $\mathcal{E}_1, \E_0$ are chosen in three different ways that correspond to the various adversaries described in the previous paragraphs. 

The \emph{random} adversary has actually no choice. He chooses the expert set $\E$ uniformly at random among all sets of size $\mu |V|$. Then $\E$ is in turn partitioned into $\E_1$ and $\E_0$ by adding each vertex in~$\E$ to~$\mathcal{E}_1$ independently with probability $1/2+\delta$  and to $\E_0$ otherwise. The \emph{weak} adversary is allowed to choose the expert set with the restriction that~$|\E|=\mu |V|$; the selected set  is then partitioned into~$\E_1$ and~$\E_0$ like in the random adversary. Finally, the \emph{strong} adversary chooses~$\E, \E_1$ and $\E_0= \E\setminus \E_1$ arbitrarily such that~$|\E|=\mu |V|, |\E_1|=(1/2+\delta)|\E|$ and consequently $|\E_0|=(1/2-\delta)|\E|$. We will ignore rounding issues througout to facilitate the presentation.

All vertices that know the truth in a graph are assigned the label `1', including all vertices in~$\E_1$, and all vertices that believe a falsehood are labeled `0'. Vertices without an opinion bear no label. The experts disseminate their opinions to the non-experts $V\setminus \E$ by a majority rule, that is, every vertex in $V\setminus \E$ takes the opinion of the majority of its neighbouring experts. To be completely explicit, a non-expert  is labeled `1'/'0' if more that half of its neighbouring experts are labeled `1'/'0'.
Vertices at which there is no majority -- because of a tie of `1's and `0's or because they have no expert neighbours -- decide upon their opinion uniformly at random, i.e., each of these vertices is independently labeled `1' with probability 1/2 and `0' otherwise.

We say that a graph is \emph{robust} against the random/ weak/ strong adversary if with high probability, for any choice of the expert set, after the dissemination process more than half of the vertices are labeled `1'. 
'With high probability' means with probability approaching 1 as $n$ approaches infinity, which we sometimes abbreviate with whp. In \cite{Alon2015} the authors studied which properties of a graph make it robust. They discovered that all graphs with maximal degree being sub-linear in $n$ are robust against the weak adversary. Furthermore, they showed that certain well-connected networks are robust against the strong adversary. In particular, such networks are either  Erd\H{o}s-R\'enyi random graphs having edge probability $p$ greater than $c/n$ for a suitable constant $c>0$, or expander graphs, with $d, \lambda_2$ being the largest and second largest eigenvalue of its adjacency matrix, satisfying
$d \ge \lambda_2/(\delta \sqrt{\mu (1-\mu +2\delta\mu)}).$
\begin{figure}[b]
\centering
\begin{tikzpicture}
\begin{scope}[every node/.style={circle,thick,draw,minimum size=0.5cm}]
    \node[fill=red] (X) at (-1,2) { };
   \node[fill=red] (A) at (0,2) { };
   \node[fill=red] (B) at (1,2) { };
   \node[fill=red] (C) at (2,2) { };
   \node[fill=red] (D) at (3,2) { };
   \node[fill=red] (E) at (4,2) { };
   \node[fill=blue] (F) at (5,2) { };
   \node[fill=blue] (G) at (6,2) { };
   \node[pattern=dots,pattern color=blue] (H) at (7,2) { };
   \node[pattern=dots,pattern color=blue] (I) at (8,2) { };
   \node[pattern=dots,pattern color=blue] (J) at (9,2) { };
   \node[pattern=dots,pattern color=blue] (K) at (10,2) { };
   \node[pattern=dots,pattern color=blue] (L) at (11,2) { };
\end{scope}
\begin{scope}[>={Stealth[black]},
              every edge/.style={draw=black, very thick}]
     \path (X) edge node {} (A);
    \path (A) edge node {} (B);
    \path  (B) edge node {} (C);
    \path  (C) edge node {} (D);
    \path  (D) edge node {} (E);
    \path  (E) edge node {} (F);
    \path  (F) edge node {} (G);
    \path  (G) edge node {} (H);
    \path  (H) edge node {} (I);
    \path  (I) edge node {} (J);
    \path  (J) edge node {} (K);
     \path  (K) edge node {} (L);
\end{scope}
\end{tikzpicture}
\begin{minipage}{1\textwidth}
~\\~\\
\end{minipage}

\begin{tikzpicture}
\begin{scope}[every node/.style={circle,thick,draw,minimum size=0.5cm}]
    \node[fill=red] (X) at (-1,2) { };
   \node[fill=red] (A) at (0,2) { };
   \node[fill=red] (B) at (1,2) { };
   \node[fill=red] (C) at (2,2) { };
   \node[fill=red] (D) at (3,2) { };
   \node[fill=red] (E) at (4,2) { };
   \node[fill=blue] (F) at (5,2) { };
   \node[pattern=dots,pattern color=blue] (G) at (6,2) { };
   \node[pattern=dots,pattern color=blue] (H) at (7,2) { };
   \node[fill=blue] (I) at (8,2) { };
   \node[pattern=dots,pattern color=blue] (J) at (9,2) { };
   \node[] (K) at (10,2) { };
   \node[] (L) at (11,2) { };
\end{scope}
\begin{scope}[>={Stealth[black]},
              every edge/.style={draw=black, very thick}]
    \path (X) edge node {} (A);
    \path (A) edge node {} (B);
    \path  (B) edge node {} (C);
    \path  (C) edge node {} (D);
    \path  (D) edge node {} (E);
    \path  (E) edge node {} (F);
    \path  (F) edge node {} (G);
    \path  (G) edge node {} (H);
    \path  (H) edge node {} (I);
    \path  (I) edge node {} (J);
    \path  (J) edge node {} (K);
     \path  (K) edge node {} (L);
\end{scope}
\end{tikzpicture}
\begin{minipage}{0.88\textwidth}
\caption{\small This figure shows an example from \cite{Alon2015}. The colors red/blue correspond to the experts labeled `1'/'0'. The dotted vertices indicate their label after the dissemination process, the unmarked vertices are decided randomly. In the first line graph we consider the iterative strong adversary, where only the rightmost blue vertex determines the label of all remaining vertices. In the second line we consider the non-iterative setting, each blue expert can at most convince two non-experts. If $1/2+\delta>3(1/2-\delta)$ and $n$ is large, the adversary can not hope to convince more that half of all vertices. 
}
\label{fig:LitRef}
\end{minipage} 
\end{figure}
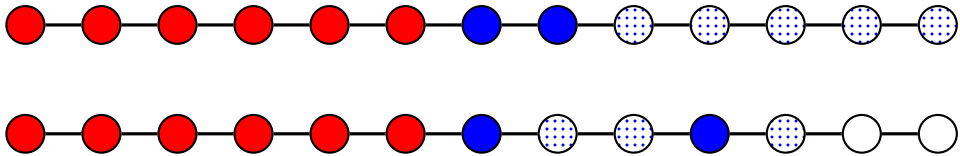
\paragraph{Iterative Dissemination}
In \cite{Alon2015} the authors also introduced an iterative version of the model with a more dynamic dissemination process. The \emph{iterative model} also starts with labeled experts and all non-experts are labeled according to the majority of their neighbouring experts. Ties that involve at least one expert are broken uniformly at random. All non-experts without any expert neighbours, however, do not form their opinion right away, but remain unlabeled. This process is then iterated by considering all vertices with label `1' and all vertices with label `0', 
until all  vertices are labeled. 

A natural question is whether iterativity helps or hinders the adversary.
Intuitively, iterativity ought to be beneficial for the adversary.
If a graph is not robust against a non-iterative adversary, then there is a choice of expert sets such that after one round of dissemination there are more vertices that are labeled '0' than '1'.
The remaining vertices without expert neighbours are then either decided randomly (non-iterative) or there are subsequent rounds of dissemination (iterative). As there now are more '0' labeled vertices that '1' labeled vertices, deciding the label of the remaining vertices by dissemination should be beneficial for the adversary.
Indeed, the authors of \cite{Alon2015} provided examples where this is the case. For example, they showed that for suitable values of~$\mu$ and~$\delta$, a line graph is robust against the non-iterative strong/weak adversaries, but not against their iterative versions, see Fig.~\ref{fig:LitRef}. 

However, in \cite{Alon2015}  an additional example, where for the weak adversary the opposite is true, was constructed. Consider a graph that is a disjoint union of a star  and a $d$-regular expander graph.
Place one expert in the center of the star and distribute the other experts as evenly as possible on the expander. In the non-iterative setting, each expert in the expander will spread its label to $d$ many non-experts. If the expert in the center of the star is labeled `0', all vertices in the star are labeled `0' as well, outweighing the difference between `1's and `0's in the expander.
In the iterative setting however, each expert does not only spread its label to $d$ many other vertices, but all vertices in the expander will be labeled at the end of the dissemination, roughly in the same ratio as that of the experts in the beginning.
Now the difference in `1's and `0's is so large that  even if all vertices in the star were labeled with `0' can sway the majority. 

Guided by the intuition described previously, it seems that no such construction can work for a strong adversary. In the previous example of the graph consisting of an expander and a star the adversary can place all '1'-labeled experts on the star  and all others in the expander. Then, all vertices in the expander will be labeled '0' resulting in a clear majority. Consequently, in~\cite{Alon2015} the  following conjecture concerning the effect of iterativity in that case was made. 
\begin{conjecture}[\cite{Alon2015}]\label{conjectureAlon}
In the case of a strong adversary an iterative propagation can never harm the adversary.
\end{conjecture}
Equivalently, the conjecture states that there is no graph 
that is robust against the iterative strong adversary and simulaneously not robust against the non-iterative version -- in this precise sense iterativity does not harm/can only help the adversary.

\paragraph{Related Results}
Besides of \cite{Alon2015}, where this model for opinion formation was introduced, there is one more work that studies questions in this precise framework. In his doctoral thesis \cite{Daknama2018}, Daknama studied resilience properties of random graphs. 'Local resilience' in this context refers to the largest number of edges, which are adjacent to any vertex, that can be removed so that the graph still is robust against the strong adversary. In~\cite{Daknama2018} it was shown that one can delete up to a fraction of $2(1-\mu+2\delta \mu)\delta/(1+2\delta)$ of all edges at each vertex without affecting robustness. 

There are also other directly related studies in opinion forming, which, however, do not use the exact model presented here. These papers include studies on word of mouth \cite{Young2009}, group recommendation \cite{Andersen2008, Grandi2016,lev2017group,Faliszewski2016} and informational cascades \cite{Bikhchandani1992, Bikhchandani1998, Watts2002, Alon2012, Feldman2014}. For further references see also \cite{Alon2015}.

\paragraph{Result} 
The contribution of this paper is to refute Conjecture \ref{conjectureAlon}. The idea is to consider a graph that has non-robustness against the non-iterative strong adversary in a very weak way. More concretely, the majority for '0' labeled vertices is only achieved if a majority of vertices without expert neighbours is labeled '0'. If we consider iterativity, then the adversary has no clear advantage in a subsequent round of the dissemination, as there are roughly equally many '1'- and '0'-labeled vertices. Additionally, we can construct the graph in a way such that the vertices without expert neighbours are connected to vertices that are labeled '1' in the first round of the dissemination, so that the adversary \emph{gets harmed}.

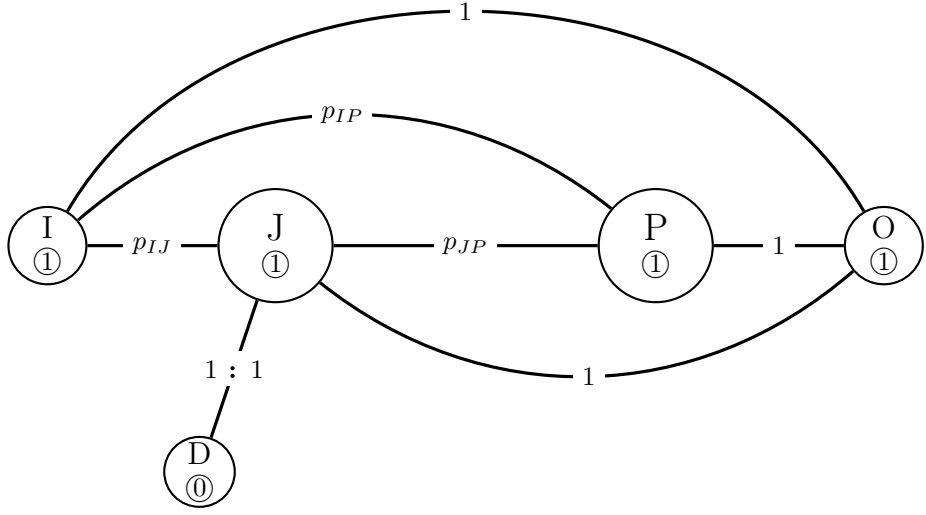
\begin{figure}[h]
\centering
\begin{tikzpicture}
\begin{scope}[every node/.style={circle,thick,draw,text width=0.5cm,align=center}]
   \node[minimum size=1cm,, inner sep=1pt] (A) at (0,2) {{\large I} \raisebox{.5pt}{\textcircled{\raisebox{-.9pt} {1}}}};
    \node[minimum size=1.5cm] (B) at (3,2) {{\Large J} \raisebox{.5pt}{\textcircled{\raisebox{-.9pt} {1}}}};
    \node[minimum size=1.5cm] (C) at (8,2) {{\Large P} \raisebox{.5pt}{\textcircled{\raisebox{-.9pt} {1}}}};
    \node[minimum size=1cm, inner sep=1pt] (D) at (11,2) {{\large O} \raisebox{.5pt}{\textcircled{\raisebox{-.9pt} {1}}}};
    \node[inner sep=0pt] (E) at (2,-1) {{\large D} \raisebox{.5pt}{\textcircled{\raisebox{-.9pt} {0}}}};
\end{scope}

\begin{scope}[>={Stealth[black]},
              every node/.style={fill=white,shape=rectangle},
              every edge/.style={draw=black, very thick}]
    \path (A) edge node {$p_{IJ}$} (B);
    \path  (B) edge node {$p_{JP}$} (C);
    \path  (A) edge[bend left=40] node {$p_{IP}$} (C);
    \path  (D) edge node {$1$} (C);
    \path  (D) edge[bend left=40] node {$1$} (B);
    \path  (A) edge[bend left=60] node {$1$} (D);
    \path  (B) edge node {$1\textbf{ : }1$} (E); 
\end{scope}
\end{tikzpicture}
\begin{minipage}{0.85\textwidth}
\caption{\small This figure shows the graph $G$. The numbers on the edges and in the vertices give the probability that an edge is present between/in the components. For example, any edge with one vertex in  $I$ and one in $J$ exists independently with probability $p_{IJ}$. Every vertex in $D$ has exactly one distinct neighbour in $J$ and no other neighbours, i.e., every vertex in $D$ has degree 1 and no two vertices in $D$ have a common neighbour. }
\label{fig:Result}
\end{minipage}
\end{figure}
Consider the following graph that implements these ideas. Let $0<\mu,\delta<1/2$ and $0< \varepsilon_1<2\delta/(1/2+\delta)$, $0<\varepsilon_2<(1/2-\delta)/(1/2+\delta)$ as well as $0<d<(1-\mu-2\delta\mu)/3$. Then the  graph $G=(V,E),\ |V|=n$ is given by $V=I~\dot\cup ~J~\dot\cup~ O~\dot\cup~ P~\dot\cup ~D$ such that 
$$|I\cup J|=|O\cup P|=(1-d)\frac{n}{2},\quad |D|=dn,\quad |I|=\mu \left(\frac{1}{2}+\delta\right)n\quad\text{and} \quad |O|=\mu \left(\frac{1}{2}-\delta\right)n.$$
The subset $D$ forms an independent set. In contrast, $I,J, O$ and $P$ each form a clique. Every vertex in $O$ is connected to all other vertices except to those in $D$. Between $I$ and $J$, $I$ and $P$ and $J$ and $P$ are random bipartite graphs with edge-probabilities $p_{IJ}$, $p_{IP}$ and $p_{JP}$ respectively. Every vertex in $D$ has degree one, with the  unique neighbor being in $J$; moreover, no two vertices in $D$ have the same neighbour. There are no more edges.
Set
$$p_{IJ}=p_{JP}=\frac{1/2-\delta}{1/2+\delta}+\varepsilon_1\qquad \text{and}\qquad \qquad p_{IP}=\frac{1/2-\delta}{1/2+\delta}-\varepsilon_2.$$
See Fig.~\ref{fig:Result} for a depiction of $G$. 

Assume for now that the adversary chooses $\E_1=I$ and $\E_0=O$. Then whp all vertices in $P$ will have $\approx\varepsilon_2|I|$
more neighbours in $O$ than in $I$ by choice of $p_{IP}$, thus they will be labeled '0' independently of iterativity.
In contrast, vertices in $J$ have $\approx\varepsilon_1|I|$ more neighbours in $I$ than in~$O$ and will consequently be labeled '1'. Summarizing, we have that all vertices in $I\cup J$ are labeled '1' and  all vertices in $O\cup P$ are labeled~'0'.
As both unions have by construction the same size, the labels of vertices in $D$ decide whether the adversary succeeds or not. This is where (non-)iterativity comes into play. In the non-iterative setting, vertices in $D$ will choose uniformly at random, as they have no neighbours in $I \cup O$. So, with positive probability there will be more vertices labeled '0' than '1' in $D$, consequently granting a majority of '0'-labeled vertices. In the iterative setting however, all vertices in $D$ will be labeled '1', as they are exclusively connected to vertices in $J$. Thus, when choosing $\E_1=I$ and $\E_0=O$ the iterative adversary fails, while the non-iterative adversary succeeds.  
Choosing the proportions of $I,O, J$ and $P$ and the edges between them suitably, we can make sure that choosing $\E_1$ and $\E_0$ differently is not advantageous for the adversary and therefore $G$ is indeed robust against the iterative strong adversary.  The main result of this paper is to show that the graph $G$ has indeed the properties outlined above. 
\begin{theorem}\label{conjecture}
For all $0<\mu<1/2$ and $1/6<\delta<1/2$ there are  $\varepsilon_1,\varepsilon_2, d>0$ such that $G$ is whp robust against the iterative strong adversary, but not against the non-iterative strong adversary.
\end{theorem}
Note that $\delta>1/6$ is a necessary constraint for our construction, but we are certain that there is an example for smaller $\delta$ as well.
Permissible values in Theorem \ref{conjecture} are, e.g.~$\mu=\delta=1/5,\ \varepsilon_1=10^{-2},\ d=10^{-4}$ and $\varepsilon_2=10^{-6}.$ 
The remainder of this paper will consist of the proof of Theorem~\ref{conjecture}. We  first state and prove a well known description of the edge distribution of random graphs and then show the claimed (non-) robustness. 

\section{Proof}
For a graph $G=(V,E)$ let $N(v)=\{w\in V\mid (v,w)\in E\}$ be the set of neighbours of $v$. We begin with a statement about the distribution of edges in random graphs.
\begin{lemma}\label{exp_lem_2}\label{exp_lem}
Let $\varepsilon>0$. The Erd\H{o}s-R\'enyi random graph $G(n,p)$  with vertex set $V$ and  $p\ge \varepsilon$ has whp the following property. For any set $S\subseteq V$ of size $|S|\ge \varepsilon n$
there is a set $X_S\subset V\setminus S$ of size 
at most ${4\varepsilon^{-3}(\ln \varepsilon^{-1} +2)}$ 
such that
$$\forall v\in (V\setminus S)\setminus X_S: \big ||N(v)\cap S|-p|S|\big| \le \varepsilon p|S|.$$
\end{lemma}
Similar versions of Lemma \ref{exp_lem} with (somehow) different bounds exist in the literature, see for example \cite[Lem.~IV.1 and IV.3]{Fountoulakis2010}. However, as we did not find the exact statement we will need in the literature we include a proof. We will utilize the following Chernoff bound.
\begin{theorem}[\cite{Arora2009}, Cor 7.11]\label{Chernoff}
Let $X$ be a binomially distributed random variable. Then
$${P}\Big(|X-\mathbb{E}[X]|>\delta\mathbb{E}[X]\Big)\le 2\exp\left(-\min\{\delta^2,\delta\}\mathbb{E}[X]/4\right ), \qquad \delta > 0.$$
\end{theorem}\begin{proof}[Proof of Lemma \ref{exp_lem_2}]
Let $S\subseteq V, |S|\ge \varepsilon n$ and let 
$$X_S=\big\{v\in V\setminus S\bigm\vert \big ||N(v)\cap S|-p|S|\big| > \varepsilon p|S|\big\}$$
be the set of vertices not satisfying the claim of the lemma. The number of neighbours of any vertex $v\in V\setminus S$ is a binomially distributed random variable, $|N(v)\cap S|\ =\text{Bin}(|S|,p)$, and the expected number of neighbours of $v$ in $S$ is $p|S|$. Thus the probability of $v\in X_S$ can be bounded with Theorem \ref{Chernoff} by
$${P}\big (\big||N(v)\cap S|-p|S|\big| > \varepsilon p|S|\big )\le \exp \left (- \varepsilon^2 p|S|/4\right ). $$
Let furthermore $t\in \mathbb{N}$; the  probability that $t$ distinct vertices are in $X_S$ is at most $\exp(-\varepsilon^2p|S|/4\cdot t)$ as the events of vertices being elements of $X_S$ are independent. There are $\binom{n}{k}\le \left (\frac{en}{k}\right )^k$ possibilities to choose $S$, a set of size $k\ge \varepsilon n$. Hence the probability that for fixed $k\ge \varepsilon n$ there is a set $S, |S|= k $ such that $|X_S|=t$ is by union bound at most \begin{align*}
\exp\left (k\ln (en/k)-\frac{\varepsilon^2pk}{4}\cdot t\right )\le \exp\left (k\left (-\ln \varepsilon +1-\frac{\varepsilon^3}{4}\cdot t\right )\right ),
\end{align*}
where we used the assumption that $p\ge \varepsilon$. Thus, if $t\ge{4{\varepsilon^{-3}}(\ln \varepsilon^{-1} +2)}$ this expression is $\le e^{-k}$ and summing over $k\ge \varepsilon n$ yields the claim.
\end{proof}

This concludes the preparations. Next we prove the main theorem, by proving the two claims separately. We show the robustness of $G$ against the iterative strong adversary first.
\begin{lemma}\label{lemma_robust}
For all $0<\mu<1/2$ and $1/6<\delta<1/2$ there are values $\varepsilon_1,\,\varepsilon_2,\, d>0$ such that $G$ is whp robust against the iterative strong adversary.
\end{lemma}
\begin{proof}
Let $0<\mu<1/2,\ 1/6<\delta<1/2$ and $\varepsilon_1, \varepsilon_{2},\ d>0$ such that 
\begin{align}\label{ass:e1}
\varepsilon_1<\min\left \{\frac{\delta\mu}2, \frac{4\delta}{1/2+\delta}-1\right\}
\end{align} and furthermore
\begin{align}\label{ass:d}
d<\min\left \{\frac{\varepsilon_1\delta}{1/2+\delta},\frac{\varepsilon_1\delta\mu}4 ,\frac{1-\mu -2\delta\mu}{3}\right \}
\end{align}
as well as
\begin{align}\label{ass:e2}
\varepsilon_2<\min \left\{\frac{d}{6}\left (\frac{4\delta}{1/2+\delta}-1-\varepsilon_1\right ),\frac{1/2-\delta}{1+2\delta}\right\}.
\end{align}
We will show that for any choice of experts, at the end of the dissemination the majority will be labeled '1' thus proving robustness. Let therefore $\E=\E_1\cup\E_0$ be any set of experts as chosen by the iterative strong adversary and define
$$i_1:=|I\cap \mathcal{E}_{1}|, \quad j_{1}:=|J\cap \mathcal{E}_{1}|, \quad o_{1}:=|O\cap \mathcal{E}_{1}|, \quad p_1:=|P\cap \mathcal{E}_{1}|, \quad d_1:=|D\cap \mathcal{E}_{1}|$$
as well as
$$i_0:=|I\cap \mathcal{E}_{0}|, \quad j_{0}:=|J\cap \mathcal{E}_{0}|, \quad o_{0}:=|O\cap \mathcal{E}_{0}|, \quad p_0:=|P\cap \mathcal{E}_{0}|, \quad d_0:=|D\cap \mathcal{E}_{0}|.$$
 By definition of the model we have that $|\mathcal{E}_{1}|=\left( {1}/{2}+\delta\right )\mu n$ as well as $|\mathcal{E}_{0}|=\left( {1}/{2}-\delta\right )\mu n$ and therefore
\begin{align}\label{eq:1}
i_1+j_1+o_1+p_1+d_1=\left(\frac{1}{2}+\delta\right )\mu n\quad\text{and}\quad i_0+j_0+o_0+p_0+d_0=\left(\frac{1}{2}-\delta\right )\mu n,
\end{align}
which readily implies that
\begin{align}\label{eq:2}
2\delta\mu n=\frac{2\delta}{1/2+\delta}(i_1+j_1+p_1+o_1+d_1).
\end{align}
We will see that the iterative dissemination will be finished after two rounds only. We start by determining the label of each vertex in the different components after the first round of dissemination.
This is decided by the  difference in `0'/`1' labeled expert neighbours. Consider the difference $$
    \Delta(v):=|N(v)\cap \mathcal{E}_{1}|-|N(v)\cap \mathcal{E}_{0}|, \quad v\in V.
$$ 
In particular, $\Delta(v)> 0$ means that $v\in V\setminus (\mathcal{E}_1\cup \mathcal{E}_0)$ will be labeled `1' and $\Delta(v)<0$ means it will be labeled `0'. Note that vertices $v$ with $\Delta(v)=0$ could be either labeled randomly (if they have the same positive number of '0'/'1' labeled neighbours) or not at all in this round. 
 
We begin with a vertex $v\in O$. Using the construction of $G$ and \eqref{eq:1} we get
\begin{align*}
|N(v)\cap \mathcal{E}_{1}|&=i_1+ j_1+ p_1+o_1=\left (\frac{1}{2}+\delta\right )\mu n-d_1
\end{align*}
and similarly 
\begin{align*}
|N(v)\cap\mathcal{E}_{0}|&=i_0+j_0+p_0+o_0=\left (\frac{1}{2}-\delta\right )\mu n-d_0.
\end{align*}
Combining these two equations, $d<\varepsilon_1<\delta\mu/2$ given by \eqref{ass:d} and \eqref{ass:e1} implies
$$\Delta(v)=2\delta\mu n-(d_1-d_0)>0\qquad \forall~v\in O.$$
We continue with $v\in J$. Using Lemma \ref{exp_lem}, we get that for all $\varepsilon>0$ whp there is $J_P\subset J,~ |J_P|\le  4\varepsilon^{-3}(\ln \varepsilon^{-1} +2)$ such that 
$$\big||N(v)\cap P\cap \E_1|-p_{JP}\cdot p_1\big|\le\varepsilon\cdot  p_{JP}\cdot p_1+ \varepsilon n \quad \text{for all } v\in J\setminus J_P.$$
As $\varepsilon>0$ is arbitrary we infer that 
$$|N(v)\cap P\cap \E_1|=p_{JP}\cdot p_1+o(n)\quad \text{for all } v\in J\setminus J_P.$$
Completely analogous calculations for $I$ and $\E_0$ yield that whp there is $J'\subset J,~ |J'|~=o(n)$ such that for all $v\in J\setminus J'$ 
\begin{align*}
|N(v)\cap \mathcal{E}_{1}|&=p_{IJ}\cdot i_1+ j_1+ p_{JP}\cdot p_1+o_1+o(n)\\&=
\left (\frac{1}{2}+\delta\right )\mu n-(1-p_{IJ})i_1-(1- p_{JP})p_1-d_1+o(n)
\end{align*}
and
\begin{align*}
|N(v)\cap\mathcal{E}_{0}|&=p_{IJ}\cdot i_0+ j_0+ p_{JP}\cdot p_0+o_0+o(n)\\&=
\left (\frac{1}{2}-\delta\right )\mu n-(1-p_{IJ})i_0-(1- p_{JP})p_0-d_0+o(n).
\end{align*}
Computing the difference of the above expressions we get for all $v\in J\setminus J'$
\begin{align*}
\Delta(v)&=2\delta\mu n-(1-p_{IJ})(i_1-i_0)-(1- p_{JP})(p_1-p_0)-(d_1-d_0)+o(n)\\
&=2\delta\mu n-\left (\frac{2\delta}{1/2+\delta}-\varepsilon_1\right )\Big ((i_1-i_0)+(p_1-p_0)\Big )-(d_1-d_0)+o(n)\\
&= 2\delta\mu n-\frac{2\delta}{1/2+\delta}\Big ((i_1-i_0)+(p_1-p_0)\Big )+\varepsilon_1\Big ((i_1-i_0)+(p_1-p_0)\Big )-(d_1-d_0)+o(n).
\end{align*}
Applying \eqref{eq:2} and \eqref{eq:1} we can obtain a lower bound for $\Delta(v),v\in J\setminus J'$
\begin{align*}
\Delta(v)&\ge \frac{2\delta}{1/2+\delta}(j_1+o_1+d_1+i_0+p_0)+\varepsilon_1\Big ((i_1-i_0)+(p_1-p_0)\Big )-d_1+o(n)\\
&\ge \varepsilon_1(i_1+j_1+p_1+o_1+d_1)+\left ( \frac{2\delta}{1/2+\delta}-\varepsilon_1\right )(i_0+p_0)-d_1.
\end{align*}
According to \eqref{ass:e1} and \eqref{ass:d} we have $\varepsilon_1<2\delta/(1/2+\delta)$ as well as $d<\varepsilon_1\delta\mu /4$ and therefore
\begin{align*}
\Delta(v)\ge \varepsilon_1\cdot 2\delta\mu n -dn>0\qquad \forall~ v\in J\setminus J'.
\end{align*}
Next we look at $v\in I.$ Using again Lemma \ref{exp_lem} and~\eqref{eq:1} we infer that whp there is $I'\subset I,~ |I'|~=o(n)$ such that for all $v\in I\setminus I'$
\begin{align*}
|N(v)\cap \mathcal{E}_{1}|&=i_1+p_{IJ}\cdot j_1+p_{IP}\cdot p_1+o_1+o(n)\\&=\left (\frac{1}{2}+\delta\right )\mu n-\left (1-p_{IJ}\right )j_1-\left (1-p_{IP}\right )p_1-d_1+o(n)
\end{align*}
and 
\begin{align*}
|N(v)\cap\mathcal{E}_{0}|&=i_0+p_{IJ}\cdot j_0+p_{IP}\cdot p_0+o_0+o(n)\\&=\left (\frac{1}{2}-\delta\right )\mu n-\left (1-p_{IJ}\right )j_0-\left (1-p_{IP}\right )p_0-d_0+o(n).
\end{align*}
By combining those bounds we obtain for all $v\in I\setminus I'$
\begin{equation}\label{eq:5}
\begin{aligned}
    \Delta(v)
    &=
    2\delta \mu n-\left (1-p_{IJ}\right )(j_1-j_0)-\left (1-p_{IP}\right )(p_1-p_0)-(d_1-d_0)+o(n) \\
    & \hspace{-6mm}=
    2\delta\mu n-\frac{2\delta}{1/2+\delta}\Big ((j_1-j_0)+(p_1-p_0)\Big )-(d_1-d_0)+\varepsilon_{1}(j_1-j_0)-\varepsilon_{2}(p_1-p_0)+o(n).
\end{aligned}
\end{equation}
Before we conclusively determine $\Delta(v)$ for $ v\in I\setminus I'$ we  look at vertices $v\in P$. Using once more Lemma~\ref{exp_lem}  and~\eqref{eq:1} we infer that whp there is $P'\subset P,~ |P'|~=o(n)$ such that for all $v\in P\setminus P'$
\begin{align*}
|N(v)\cap \mathcal{E}_{1}|&=p_{IP}\cdot i_1+p_{JP}\cdot j_1+ p_1+o_1+o(n)\\&=\left (\frac{1}{2}+\delta\right )\mu n-\left (1-p_{IP}\right )i_1-\left (1-p_{JP}\right )j_1-d_1+o(n)
\end{align*}
and 
\begin{align*}
|N(v)\cap\mathcal{E}_{0}|&=p_{IP}\cdot i_0+p_{JP}\cdot j_0+ p_0+o_0+o(n)\\&=\left (\frac{1}{2}-\delta\right )\mu n-\left (1-p_{IP}\right )i_0-\left (1-p_{JP}\right )j_0-d_0+o(n).
\end{align*}
Together these two expressions yield for all $v\in P\setminus P'$
\begin{align}\label{eq:4}
\Delta(v)&=  2\delta\mu n-\frac{2\delta}{1/2+\delta}\Big((i_1-i_0)+(j_1-j_0)\Big)-(d_1-d_0)+\varepsilon_{1}(j_1-j_0)-\varepsilon_{2}(i_1-i_0)+o(n).
\end{align}
We argue next, that either ``$\Delta(v)<0 \text{ for some }v\in I\setminus I'\,$'' or ``$\Delta(v)<0 \text{ for some }v\in P\setminus P'\,$'' but never both. To see this, observe that 
$$\text{``}\Delta(v)<0 \text{ for some }v\in I\setminus I'   \quad\text{and}\quad \Delta(v)<0 \text{ for some }v\in P\setminus P \, \text{''}$$ implies that 
\begin{align}\label{eq:3}
(i_1-i_0)+(j_1-j_0) \quad \text{and}\quad (j_1-j_0)+(p_1-p_0)  \quad \text{are both}\quad \ge((1/2+\delta)\mu-\varepsilon_1 )n.
\end{align}
Otherwise~\eqref{ass:d} and \eqref{ass:e2} assert  that $\varepsilon_2< d/6$ as well as $d< \varepsilon_1\delta/(1/2+\delta)$ and therefore either by \eqref{eq:5}
\begin{align*}
\Delta(v)&\ge  \frac{2\delta}{1/2+\delta}\varepsilon_1 n-dn-\varepsilon_{2}n>0,\qquad\text{for all }  v\in I\setminus I'
\end{align*}
or by \eqref{eq:4}
\begin{align*}
\Delta(v)&\ge  \frac{2\delta}{1/2+\delta}\varepsilon_1 n-dn-\varepsilon_{2}n>0,\qquad \text{for all } v\in P\setminus P'.
\end{align*}
However, as $(i_1-i_0)+(p_1-p_0)\le (1/2+\delta)\mu n$ we obtain from~\eqref{eq:3}  that $j_1-j_0\ge (\delta\mu -\varepsilon_1)n$. Then~\eqref{ass:e1},\eqref{ass:d} and \eqref{ass:e2} imply that $\varepsilon_1< \delta\mu/2, \, d<\varepsilon_1\delta\mu/$ and $\varepsilon_2< d/6$ and thus  \eqref{eq:5} yields
\begin{align*}
\Delta(v)&\ge  \varepsilon_1\cdot (\delta\mu-\varepsilon_1 )n-dn- \varepsilon_{2}n>0,\qquad \text{for all } v\in I\setminus I'.
\end{align*}
Summarizing, we have shown that $\Delta(v)>0$ for all $v\in (J\setminus J')\cup O$ and either $\Delta(v)>0$ for all $v\in I\setminus I'$ or $\Delta(v)>0$ for all $v\in P\setminus P'$.

In the rest of the proof we consider the second round of the iterative dissemination process. We will distinguish two cases. Assume first that $j_0+ d_0< (d/2-\varepsilon_2)n$. As $\Delta(v)>0$ for all $v\in J\setminus J'$ we infer that at most $(d/2-\varepsilon_2)n+o(n)$ vertices in $D$ will be labeled `0' after the second round of the  dissemination process, all other vertices in $D$ will be labeled `1'. Thus counting the total number of vertices labeled `1' after the process, we get in this case for $n$ large enough
\begin{align*}
\#(\text{vertices labeled `1'})&>|I\setminus I'|\ +\ |J\setminus J'|\ +\ |O|\ +\left (\frac{d}{2}+\varepsilon_2-o(1)\right) n-|\mathcal{E}_0|\\
&=(1-d)\frac{n}{2}+\frac{dn}{2}+\varepsilon_2n-o(n)>\frac{n}{2}.
\end{align*}
We are left with the case $j_0+d_0\ge (d/2-\varepsilon_2)n$. Observe that $d_1<(d/2+\varepsilon_2)n$ as otherwise the conclusion of the previous case applies. We revisit $\Delta(v),\ v\in P\setminus P'$ using \eqref{eq:4} and \eqref{eq:2} 
\begin{align*}
\Delta(v)&=\frac{2\delta}{1/2+\delta}(p_1+o_1+d_1+i_0+j_0)-(d_1-d_0)+\varepsilon_{1}(j_1-j_0)-\varepsilon_{2}(i_1-i_0)+o(n)\\
&\ge \left (\frac{2\delta}{1/2+\delta}-\varepsilon_1\right )j_0+d_0 - \left (1-\frac{2\delta}{1/2+\delta}\right ) d_1-\varepsilon_{2}i_1+o(n).
\end{align*}
Using the assumptions $j_0+d_0\ge (d/2-\varepsilon_2)n$ and $d_1<(d/2+\varepsilon_2)n$, this simplifies to 
\begin{align*}
\Delta(v)&> \left (\frac{4\delta}{1/2+\delta}-1-\varepsilon_1\right )dn/2-3\varepsilon_{2}n+o(n).
\end{align*}
Assumption \eqref{ass:e2} guarantees that $\Delta(v)>0,\  v\in P\setminus P'$ and thus in this case for $n$ large enough
\begin{align*}
\#( \text{vertices labeled '1'})&>|I\setminus I'|\ +\ |J\setminus J'|\ +\ |O|\ +\ |P\setminus P'|\ -\ |\mathcal{E}_0|\\
&=(1-d)n-\left (\frac{1}{2}-\delta\right )\mu n-o(n)>\frac{n}{2},
\end{align*}
and the proof is completed.
\end{proof}

The next lemma together with Lemma \ref{lemma_robust} implies Theorem \ref{conjecture}.
\begin{lemma}\label{lemma_non_robust}
For all $0<\mu<1/2$ and $1/6<\delta<1/2$ there are values $\varepsilon_1,\varepsilon_2, d>0$ such that whp $G$ is not robust against the non-iterative strong adversary.
\end{lemma}
\begin{proof}
Let  $0<\mu<1/2$ and $1/6<\delta<1/2$ and $\varepsilon_1,\, \varepsilon_2,\, d>0$ as given in \eqref{ass:e1} to \eqref{ass:e2}. 
We show that $G$ is indeed not robust by giving a suitable choice of the expert set. Set $\E=\E_1\cup\E_0$ with $\mathcal{E}_{1}=I$ and $\mathcal{E}_{0}=O$. By definition, these sets have matching cardinalities. 
We compute the quantity $$\Delta(v)=|N(v)\cap \E_1|-|N(v)\cap \E_2|$$ for vertices $v\in P$ to find their labels.
Using $\varepsilon_2< (1/2-\delta)/(1+2\delta)$ by \eqref{ass:e2} and Theorem \ref{Chernoff} we readily obtain that whp for all $v\in P$
\begin{align*}
\Delta(v)\le (p_{IP}+ o(1))|\E_1|-|\E_0|=\left (\frac{1/2-\delta}{1/2+\delta}-\varepsilon_{2}+ o(1)\right )|\E_1|-|\E_0|=-(\varepsilon_2+o(1))|\E_1|<0.
\end{align*}
Therefore the set of vertices labeled `0' contains $O\cup P$ which has cardinality $(1-d)\frac{n}{2}$. However, vertices in $D$ do not have any expert neighbours and as we are in the non-iterative setting, those vertices will be decided uniformly at random. Hence with probability $1/2$ there will be at least $dn/2+1$ vertices  labeled `0' in $D$ and therefore $G$ is not robust against the non-iterative strong adversary.
\end{proof}
\phantomsection
\bibliographystyle{abbrv}

\end{document}